\documentclass[a4paper, 12pt]{article}
\usepackage{amsmath, amsthm, amssymb, mathrsfs, graphicx, ifthen, url}
\usepackage{fullpage}
\usepackage{algorithm}

\usepackage[sort&compress]{natbib}
\bibpunct{(}{)}{;}{a}{}{,}

\usepackage{color}

\theoremstyle{definition}
\newtheorem{definition}{Definition}

\theoremstyle{plain}
\newtheorem{theorem}{Theorem}
\newtheorem{lemma}{Lemma}
\newtheorem*{corollary}{Corollary}

\newcommand{\A}{\mathscr{A}}

\newcommand{\E}{\mathsf{E}}
\newcommand{\prob}{\mathsf{P}}

\newcommand{\dhat}{\hat d}
\newcommand{\Bbar}{\overline{B}}

\newcommand{\bigmid}{\; \Bigl\vert \;}
\newcommand{\sleq}{\leq_{\text{\rm st}}}

\newcommand{\eps}{\varepsilon}
\renewcommand{\phi}{\varphi}

\title{On $\eps$-optimality of the pursuit learning algorithm}
\author{Ryan Martin \\ Department of Mathematics, Statistics, and Computer Science \\ University of Illinois at Chicago \\ \url{rgmartin@math.uic.edu} \\ \mbox{} \\ Omkar Tilak \\ Department of Computer and Information Sciences \\ Indiana University--Purdue University Indianapolis \\ \url{otilak@cs.iupui.edu}}
\date{\today}

\begin{document}

\maketitle

\begin{abstract}
Estimator algorithms in learning automata are useful tools for adaptive, real-time optimization in computer science and engineering applications.  This paper investigates theoretical convergence properties for a special case of estimator algorithms---the pursuit learning algorithm.  In this note, we identify and fill a gap in existing proofs of probabilistic convergence for pursuit learning.  It is tradition to take the pursuit learning tuning parameter to be fixed in practical applications, but our proof sheds light on the importance of a vanishing sequence of tuning parameters in a theoretical convergence analysis.      

\smallskip

\end{abstract}

\section{Introduction}
\label{S:intro}

A learning automaton consists of an adaptive learning agent operating in unknown random environment \citep{narendra.thathachar.1989}.  In a nutshell, a learning automaton has a choice among a finite set of actions to take, with one such action being optimal in the sense that it has the highest probability of producing a reward from the environment.  This optimal action is unknown and the automaton uses feedback from the environment to try to identify the optimal action.  Applications of learning automata include game theory, pattern recognition, computer vision, and routing in communications networks.  Recently, learning automata have been used for call routing in ATM networks \citep{atlasis.loukas.vasilakos.2000}, multiple access channel selection \citep{zhong.xu.tao.2010}, congestion avoidance in wireless networks \citep{misra.tiwari.obaidat.2009}, channel selection in radio networks \citep{tuan.tong.premkumar.2010}, modeling of students' behavior \citep{oommen.hashem.2010}, clustering and backbone formation in ad-hoc wireless networks \citep{torkestania.meybodi.2010a, torkestania.meybodi.2010b}, power system stabilizers \citep{kashki.abido.abdel.2010}, and spectrum allocation in cognitive networks \citep{lixia.gang.ming.yuxing.2010}.  The simplest type of learning automata applies a direct algorithm, such as the linear reward-inaction algorithm \citep{narendra.thathachar.1989}, which uses only the environmental feedback at iteration $t$ to update the preference ordering of the actions.  A drawback to using direct algorithms is their slow rate of convergence.  Attention recently has focused on the faster indirect estimator algorithms.  What sets indirect algorithms apart from their direct counterparts is that they use the entire history of environmental feedback, i.e., from iteration 1 to $t$, to update the action preference ordering at iteration $t$.  It is this more efficient use of the environmental feedback which leads to faster convergence.  

Here we consider a special case of indirect estimator algorithms---the \emph{pursuit learning algorithm}---and, in particular, the version presented by \citet{rs1996}.  Starting with vacuous information about the unknown reward probabilities, pursuit learning adaptively samples actions and tracks the empirical reward probabilities for each action.  As the algorithm progresses, the sampling probabilities for the set of actions are updated in a way consistent with the relative magnitudes of the empirical reward probabilities; see Section~\ref{SS:notation}.  Simulations demonstrate that the algorithm is fast to converge in a number of different estimation scenarios \citep{sastry1985, lanctotoommen1992, oommenlanctot1990, thathachar.sastry.1985}.  Theoretically, the algorithm is said to converge if the sampling probability for the action with the highest reward probability becomes close to 1 as the number of interations increases.  

In the learning automata literature, $\eps$-optimality is the gold standard for theoretical convergence.  But there seems to be two different notions of $\eps$-optimality that appears in the estimator algorithm literature.  The version that appears in the direct estimator context (e.g., the linear reward-inaction algorithms) is in some sense weaker than that which appears in the indirect algorithm context.  The latter is essentially convergence in probability of the dominant action sampling probability to 1 as the number of iterations approaches infinity.  Section~\ref{SS:conv} describes these two modes of stochastic convergence in more detail, but our focus is on the latter convergence in probability version.

The main goal of this paper is to identify and fill a gap in existing proofs of $\eps$-optimality for pursuit learning.  We believe that it is important to throw light on this gap because there are relatively recent papers proposing new algorithms that simply copy verbatim these incomplete arguments.  Specifically, in many proofs, the weak law of large numbers is incorrectly interpreted as giving a bound on the probability that the sample path stays inside a fixed neighborhood of its target \emph{forever} after some fixed iteration.  It is true that any finite-dimensional properties of the sample path can be handled via the weak law of large numbers, but the word ``forever'' implies that countably many time instances must be dealt with and, hence, more care must be taken.  A detailed explanation of the gap in existing proofs is presented in Section~\ref{SS:oldproofs}. In Section~\ref{S:analysis} we give a new proof of convergence in probability for pursuit learning with some apparently new arguments.  A further consequence of our analysis relates to the algorithm's tuning parameter.  Indeed, it standard to assume, in both theory and practice, that the algorithm's tuning parameter is a small but fixed quantity.  However, our analysis suggests that it is necessary to consider a sequence of tuning parameters that vanish at a certain rate.  


\section{Pursuit learning algorithm}
\label{S:pl}

\subsection{Notation and statement of the algorithm}
\label{SS:notation}

Suppose a learning automaton has a finite set of actions $A = \{a_1,\ldots,a_r\}$.  If the automaton plays action $a_i$, then it earns a reward with probability $d_i$; otherwise, it gets a penalty.  An estimator algorithm tracks this reward/penalty information with the goal if identifying the optimal action---the one having the largest reward probability $d$.  Pursuit learning, described below, is one such algorithm.

At iteration $t$, the automaton selects an action $\alpha(t) \in A$ with respective probabilities $\pi(t) = \{\pi_1(t),\ldots,\pi_r(t) \}$.  When this action is played, the environment produces an outcome $X(t) \in \{0,1\}$ that satisfies
\[ d_i = \E\{X(t) \mid \alpha(t) = a_i\}, \quad i=1,\ldots,r. \]
As the algorithm proceeds and the various actions are tried, the automaton acquires more and more information about the $d$'s indirectly through the $X$'s.  In other words, estimates $\dhat(t)$ of $d$ at time $t$ can be used to update the sampling probabilities $\pi(t)$ in such a way that those actions with large $\dhat(t)$ are more likely to be chosen again in the next iteration.  Algorithm~\ref{algo:pl} gives the details.  

\begin{algorithm*}[ht]
\smallskip
\begin{enumerate}
\item For $i=1,\ldots,r$, set
\[ \pi_i(0) = 1/r \quad \text{and} \quad N_i(0) = 0, \]
and initialize $\dhat_i(0)$ by playing action $a_i$ a few times and recording the proportion of rewards.  Set $t=1$.  
\item
\begin{enumerate}
\item Sample $\alpha(t)$ according to $\pi(t-1)$, and observe $X(t)$ drawn from its conditional distribution given $\alpha(t)$.
\item For $i=1,\ldots,r$, update
\[ N_i(t) = \begin{cases} N_i(t-1) + 1 & \text{if $\alpha(t) = a_i$} \\ N_i(t-1) & \text{if $\alpha(t) \neq a_i$,} \end{cases} \]
which denotes the number of times action $a_i$ has been tried up to and including iteration $t$, and
\[ \dhat_i(t) = \begin{cases} \dhat_i(t-1) + \frac{X(t) - \dhat_i(t-1)}{N_i(t)} & \text{if $\alpha(t) = a_i$} \\ \dhat_i(t-1) & \text{if $\alpha(t) \neq a_i$,} \end{cases} \]
and then compute
\[ m(t) = \arg\max \{ \dhat_1(t),\ldots,\dhat_r(t) \}. \]
\item Update
\[ \pi(t) = (1-\lambda) \pi(t-1) + \lambda \delta_{m(t)}, \]
where $\delta_j$ is an $r$-vector whose $j^{\text{th}}$ entry is 1 and the others 0.
\end{enumerate}
\item Set $t \leftarrow t+1$ and return to Step 2.
\end{enumerate}
\caption{\bf-- Pursuit Learning.}
\label{algo:pl}
\end{algorithm*}

For comparison, the direct linear reward-inaction algorithm updates $\pi(t)$ according to the following rule: If $\alpha(t) = a_i$, then 
\[ \pi_j(t) = \begin{cases}
\pi_j(t-1) + \lambda X(t) [1-\pi_j(t-1)] & \text{if $j = i$,} \\
\pi_j(t-1) - \lambda X(t) \pi_j(t-1) & \text{if $j \neq i$.} 
\end{cases} \]
It is clear that this direct linear reward-inaction algorithm does not make efficient use of the full environmental history $X(1),\ldots,X(t)$ up to and including iteration $t$.  For this reason, it suffers from slower convergence than that of the indirect pursuit learning algorithm.  In fact, \citet{thathachar.sastry.1985} demonstrate, via simulations, that an indirect algorithm requires roughly 87\% fewer iterations than a direct algorithm to achieve the same level of precision.  

One might also notice that the pursuit learning algorithm is not unlike the popular stochastic approximation methods introduced in \citet{robbinsmonro} and discussed in detail in \citet{kushner}.  But a convergence analysis of pursuit learning using the powerful ordinary differential equation techniques seems particularly challenging due to the discontinuity of the $\delta_{m(t)}$ component in the Step 2(c) update.  

The internal parameter $\lambda$ controls the size of steps that can be made in moving from $\pi(t-1)$ to $\pi(t)$.  In general, small values of $\lambda$ correspond to slower rates of convergence, and vice versa.  In our asymptotic results, we follow \citet{tmm2010} and actually take $\lambda = \lambda_t$ to change with $t$.  They argue that a changing $\lambda$ is consistent with the usual notion of convergence (see also Section~\ref{SS:conv}), and does not necessarily conflict with the practical choice of small fixed $\lambda$.  In what follows, we will assume that
\begin{equation}
\label{eq:lambda}
\lambda_t = 1-\theta^{1/t}, \quad \text{for some fixed $\theta \in (e^{-1},1)$},
\end{equation}
although all that is necessary is that $\lambda_t \asymp 1-\theta^{1/t}$ as $t \to \infty$.

\subsection{Convergence and $\eps$-optimality}
\label{SS:conv}

Convergence of an estimator algorithm like pursuit learning implies that, eventually, the automaton will always play the optimal action.  In other words, if $d_1$ is the largest among the $d$'s, then $\pi_1(t)$ gets close to 1, in some sense, as $t \to \infty$.  This convergence is typically called $\eps$-optimality, although there appears to be no widely agreed upon definition.

In the context of indirect estimator algorithms, the following is perhaps the most common definition of $\eps$-optimality.  We shall henceforth assume, without loss of generality, that action $a_1$ is the unique dominant action, i.e., $d_1$ is the largest of the $d$'s.

\begin{definition}
\label{def:eps.opt}
The pursuit learning algorithm is $\eps$-optimal if, for any $\eps,\delta > 0$, there exists $T^\star = T^\star(\eps,\delta)$ and $\lambda^\star = \lambda^\star(\eps,\delta)$ such that
\begin{equation}
\label{eq:eps.opt}
\prob\{ \pi_1(t) > 1-\eps\} > 1-\delta,
\end{equation}
for all $t > T^\star$ and $\lambda < \lambda^\star$.  Simply put, the algorithm has the $\eps$-optimality property if $\pi_1(t) \to 1$ in probability as $(t,\lambda) \to (\infty, 0)$.
\end{definition}

This is the definition of $\eps$-optimality that appears in \citet{agache.oommen.2002} and the references mentioned in Section~\ref{S:intro}; \citet{thathachar.sastry.1985} say an algorithm that satisfies Definition~\ref{def:eps.opt} is \emph{optimal in probability}, arguably a better adjective.  However, a different notion of $\eps$-optimality can be found in other contexts.  This one says that the algorithm is $\eps$-optimal if, for any $\eps > 0$, there exists a fixed $\lambda > 0$ so that $\lim\inf_{t \to \infty} \pi_1(t) > 1-\eps$ with probability~1.  Compared to Definition~\ref{def:eps.opt}, this latter definition is, on one hand, stronger because the condition is ``with probability~1'' but, on the other hand, weaker because it does not even require $\pi_1(t)$ to converge.  Since one will not, in general, imply the other, it is unclear which definition is to be preferred.  \citet{oommenlanctot1990} and others have recognized the difference between the two, but apparently no explanation has been given for choosing one over the other.  

Since both $T^\star$ and $\lambda^\star$ in Definition~\ref{def:eps.opt} are linked together through the choice of $(\eps,\delta)$, it is intuitively clear that $\lambda$ should decrease with $t$.  In fact, allowing $\lambda$ to change with $t$ appears to be necessary in the proof presented in Section~\ref{SS:main}.  So, throughout this paper, our notion of $\eps$-optimality will be that \eqref{eq:eps.opt} holds for all $t > T^\star$ with the particular (vanishing) sequence of tuning parameters $\{\lambda_t\}$ in \eqref{eq:lambda}.

\subsection{Existing proofs of $\eps$-optimality}
\label{SS:oldproofs}

Here we shall identify the gap in existing proofs of $\eps$-optimality for pursuit learning.  Focus will fall primarily on the proof in \citet{rs1996}, but this is just for concreteness and not to single out these particular authors.  In fact, essentially the same gap appears in \citet{papa2004}; there is a similar mis-step in other papers which we mention briefly below.  The outline of these proofs goes roughly as follows:

\begin{description}
\item[\sc Step~1.] Show that $N_i(t) \to \infty$ in probability for each $i=1,\ldots,r$ as $t \to \infty$.  That is, show that for any large $n$ and small $\delta$, there exists $T^\star$ such that
\begin{equation}
\label{eq:step1}
\prob\Bigl\{\min_{i=1,\ldots,r} N_i(t) > n \Bigr\} > 1-\delta, \quad \forall \; t > T^\star.
\end{equation}

\vspace{-1mm}

\item[\sc Step~2.] Show that for any small $\delta$ and $\rho$, there exists $n$ such that
\begin{equation}
\label{eq:step2}
\prob\Bigl\{ \max_{i=1,\ldots,r} |\dhat_i(t)-d_i| < \rho \bigmid \min_{i=1,\ldots,r} N_i(t) > n \Bigr\} > 1-\delta.
\end{equation}
\citet{rs1996} apply the famous inequality of \citet{hoeffding1963} to get an expression on the right-hand side that approaches 1 exponentially fast in $n$.  A similar idea is used in Section~\ref{SS:main}.

\vspace{-1mm}

\item[\sc Step~3.] Reason from \eqref{eq:step2} that, for sufficiently small $\rho$ and \emph{for all $t$ larger than some $T^\star$}, $\dhat_1(t)$ will be the largest among the $\dhat_i(t)$'s with probability at least $1-\delta$.

\vspace{-1mm}

\item[\sc Step~4.] Apply the monotonicity property \citep{lanctotoommen1992} to show that $\pi_1(t)$ increases monotonically to 1 starting from some $t > T^\star$ and must, therefore, eventually cross the $1-\eps$ threshold.
\end{description}

The trouble with this line of reasoning emerges in Step~3, and is a consequence of an incorrect interpretation of the law of large numbers.  Roughly speaking, what is needed in Step~3 is a control of the entire $\dhat(t)$ process over an infinite time horizon, $t > T^\star$, but the law of large numbers alone can provide control at only finitely many time instances.  More precisely, from Steps~1 and 2 and the law of large numbers one can reason that
\begin{equation}
\label{eq:fixed.time}
\prob\bigl\{ \text{$\dhat_1(t)$ is the largest of the $\dhat(t)$'s} \bigr\} > 1-\delta, \quad \forall \; t \geq T^\star.
\end{equation}
But even though the left-hand side above is monotone increasing in $t$, one cannot conclude directly from this fact that
\begin{equation}
\label{eq:infinite.time}
\prob\bigl\{ \text{$\dhat_1(t)$ is the largest of the $\dhat(t)$'s \emph{for all} $t \geq T^\star$} \bigr\} > 1-\delta.
\end{equation}
\citet{rs1996} implicitly assume that \eqref{eq:fixed.time} implies \eqref{eq:infinite.time} in their proof of $\eps$-optimality.  A slightly different oversight is made in \citet{thatacharsastry1987}, \citet{oommenlanctot1990}, and \citet{lanctotoommen1992}.  They assume that
\[ \prob\bigl\{\pi_1(t) > 1-\eps \mid \text{$\dhat_1(t)$ is the largest of the $\dhat(t)$'s} \bigr\} \]
can be made arbitrarily close to 1 for large enough $t$.  However, the knowledge that $\dhat_1(t)$ is the largest of the $\dhat(t)$'s only at time $t$ provides no control over how close $\pi_1(t)$ is to 1.  The monotonicity property in Step~4 requires that the $\dhat(t)$'s be properly ordered \emph{forever}, not just at a single point in time.

It will be insightful to have a clearer picture of what the problem is mathematically.  First, the left-hand side of \eqref{eq:infinite.time} is, in general, much smaller than the left-hand side of \eqref{eq:fixed.time}, so the claim ``\eqref{eq:fixed.time} $\Rightarrow$ \eqref{eq:infinite.time}'' immediately seems questionable.  In fact, if $E_t$ is the event that $\dhat_1(t)$ is the largest of the $\dhat(t)$'s at time $t$, then from \eqref{eq:fixed.time} we can conclude that 
\begin{equation}
\label{eq:one.way}
\liminf_{t \to \infty} \prob\{E_t\} \geq 1-\delta. 
\end{equation}
But the event inside $\prob\{\cdots\}$ in \eqref{eq:infinite.time} is
\[ \bigcap_{t \geq T^\star} E_t \subset \bigcup_{T^\star \geq 1} \bigcap_{t \geq T^\star} E_t =: \liminf_{t \to \infty} E_t, \]
and it follows from Fatou's lemma that
\begin{equation}
\label{eq:other.way}
\text{left-hand side of \eqref{eq:infinite.time}} \leq \prob\bigl\{ \liminf_{t \to \infty} E_t \bigr\} \leq \liminf_{t \to \infty} \prob\{E_t\}.
\end{equation}
So, from \eqref{eq:one.way} and \eqref{eq:other.way}, we can conclude only that the left-hand side \eqref{eq:infinite.time} is bounded from above by something greater than $1-\delta$ and, hence, \eqref{eq:fixed.time} need not imply \eqref{eq:infinite.time}.  Therefore, some pursuit learning-specific considerations are needed and, to the authors' knowledge, there is no obvious way to fill this gap.  In the next section we give a proof of $\eps$-optimality based on some apparently new arguments.

\section{A refined analysis of pursuit learning}
\label{S:analysis}

\subsection{An infinite-series result}
\label{SS:series}

Here we state an infinite-series result which will be useful in our analysis in Section~\ref{SS:main}.  For completeness, a proof is given in Appendix~\ref{S:proof}.

\begin{lemma}
\label{lem:sum}
Given $a,b \in (0,1)$, let $\zeta(t) = (1-a/t^b)^t$, $t \geq 1$.  Then $\sum_{t=1}^\infty \zeta(t) < \infty$.  
\end{lemma}

It is easy to see that the condition $b \in (0,1)$ is necessary.  Indeed, if $b > 1$, then the sequence itself converges to $e^{-a}$ and the series cannot hope to converge.  In our pursuit learning application below, this condition will be taken care of in our choice of tuning parameter sequence $\lambda_t$.  

\subsection{Main results}
\label{SS:main}

We start by summarizing a few known results from the literature \citep[see, e.g.,][]{tmm2010} which will be needed in the proof of the main theorem.  Recall the notation $N_i(t)$ used for the number of times, up to iteration $t$, that action $i$ has been tried, $i=1,\ldots,r$.  The first result is that all of the $N(t)$'s are unbounded in probability as the number of iterations $t$ increases to $\infty$.

\begin{lemma}
\label{lem:counts}
Suppose $\lambda_t$ satisfies \eqref{eq:lambda} with $e^{-1} < \theta < 1$.  Then for any small $\delta > 0$ and any $K > 0$, there exists $T_1^\star$ such that, for each $i=1,\ldots,r$,
\[ \prob\{ N_i(t) \leq K \} < \delta, \quad \forall \; t > T_1^\star. \]
\end{lemma}

As the number of times each action is played is increasing to $\infty$, it is reasonable to think that the estimates, namely the $\dhat(t)$'s, should be approaching their respective targets, the $d$'s.  It turns out that this intuition is indeed correct.

\begin{lemma}
\label{lem:nbhd}
Suppose $\lambda_t$ satisfies \eqref{eq:lambda} with $e^{-1} < \theta < 1$.  Then for any small $\delta > 0$ and any small $\eta > 0$, there exists $T_2^\star$ such that, for each $i=1,\ldots,r$,
\[ \prob\bigl\{ |\dhat_i(t) - d_i| > \eta \bigr\} < \delta, \quad \forall \; t > T_2^\star. \]
\end{lemma}

An alternative way to phrase the previous two lemmas is that, under the stated conditions, $N_i(t)$ and $\dhat_i(t)$ converge in probability to $\infty$ and $d_i$, respectively, as $t \to \infty$.  Next is the main $\eps$-optimality result.  

\begin{theorem}
\label{thm:eps.opt}
Suppose $\lambda_t$ satisfies \eqref{eq:lambda} with $e^{-1} < \theta < 1$.  Then for any small $\eps,\delta > 0$, there exists $T^\star$ such that
\[ \prob\{ \pi_1(t) > 1-\eps \} > 1-\delta, \quad \forall \; t > T^\star. \]
\end{theorem}

To prove Theorem~\ref{thm:eps.opt}, we shall initially follow the argument of \citet{rs1996}.  To simplify notation, define the events
\[ A_\eps(t) = \{\pi_1(t) > 1-\eps\}, \quad t \geq 1. \]
Next, we observe that, since the reward probabilities are fixed, there is a number $\eta > 0$ such that, if $|\dhat_1(t)-d_1| < \eta$, then $\dhat_1(t)$ must be the largest of the estimates $\dhat(t)$ at iteration $t$.  For this $\eta$, define the two sequences of events
\begin{align*}
B(t) & = \bigl\{ |\dhat_1(t)-d_1| < \eta \bigr\}, \quad t > 0, \\
\Bbar(T) & = \bigl\{\sup_{t \geq T} |\dhat_1(t)-d_1| < \eta \bigr\} = \bigcap_{t \geq T} B(t), \quad T > 0.
\end{align*}
Then, for any positive integers $t$ and $T$, the law of total probability gives
\[ \prob\{A_\eps(t+T)\} \geq \prob\{ A_\eps(t+T) \mid \Bbar(T) \} \prob\{ \Bbar(T) \}. \]
Moreover, from the monotonicity property of pursuit learning, it follows that there exists $T_3^\star$ such that
\[ \prob\{ A_\eps(t+T) \mid \Bbar(T) \} = 1, \quad \forall \; t > T_3^\star, \]
Therefore, to complete the proof, it remains to show that there exists $T > 0$ such that $\prob\{\Bbar(T)\} > 1-\delta$.  But by DeMorgan's law,
\[ \prob\bigl\{ \Bbar(T) \bigr\} = \prob\Bigl\{ \bigcap_{t \geq T} B(t) \Bigr\} = 1-\prob\Bigl\{ \bigcup_{t \geq T} B(t)^c \Bigr\}, \]
so we are done if we can find $T > 0$ such that
\begin{equation}
\label{eq:goal}
\prob\Bigl\{ \bigcup_{t \geq T} B(t)^c \Bigr\} < \delta.
\end{equation}
Towards this, write $N(t) = N_1(t)$ and note that
\begin{align*}
\prob\Bigl\{ \bigcup_{t \geq T} B(t)^c \Bigr\} & \leq \sum_{t \geq T} \prob\{B(t)^c\} \\
& = \sum_{t \geq T} \Bigl( \sum_{n=0}^t \prob\{ B(t)^c \mid N(t) = n \} \prob\{N(t) = n\} \Bigr).
\end{align*}
It follows easily from Hoeffding's inequality that
\[ \prob\{ B(t)^c \mid N(t) = n \} = \prob\bigl\{ |\dhat_1(t) - d_1| \geq \eta \mid N(t)=n \bigr\} \leq e^{-hn}, \]
where $h = \eta^2/8 > 0$ is a constant independent of $n$.  Therefore,
\begin{align*}
\prob\Bigl\{ \bigcup_{t \geq T} B(t)^c \Bigr\} & \leq \sum_{t \geq T} \Bigl( \sum_{n=0}^t \prob\{ B(t)^c \mid N(t) = n \} \prob\{N(t) = n\} \Bigr) \\
& \leq \sum_{t \geq T} \Bigl( \sum_{n=0}^t e^{-hn} \prob\{N(t) = n\} \Bigr),
\end{align*}
and the inner-most sum is easily seen to be the moment generating function, call it $\psi_t(u)$, of the random variable $N(t)$ evaluated at $u=-h$.  To prove that this sum is finite, we must show that $\psi_t(-h)$ vanishes sufficiently fast in $t$.

Formulae for moment generating functions of standard random variables are readily available.  But $N(t)$ is not a standard random variable; it is like a Bernoulli convolution \citep{proschan.sethuraman.1976, klenke.mattner.2010} but the summands are only conditionally Bernoulli.  In Lemma~\ref{lem:mgf} below we show that $\psi_t(u)$, for $u \leq 0$, is bounded above by a certain binomial random variable's moment generating function.  

\begin{lemma}
\label{lem:mgf}
Consider a binomial random variable with parameters $(t, \omega_t)$, where $\omega_t = \pi_1(0) \theta^{\gamma(t)}$ and $\gamma(t) = \sum_{s=1}^t s^{-1}$.  If $\phi_t$ is the corresponding moment generating function, then $\psi_t(u) \leq \phi_t(u)$ for $u \leq 0$.  
\end{lemma}

\begin{proof}
Let $N(t) = \sum_{s=1}^t \xi(s)$, where $\xi(s) = 1$ if the optimal action is sampled at iteration $s$ and 0 otherwise.  If $\A_s$ denotes the $\sigma$-algebra generated by the history of the algorithm up to and including iteration $s$, then $\xi(s)$ satisfies 
\[ \E\{ \xi(s) \mid \A_{s-1}\} = \pi(s-1). \]
For $u \leq 0$, the moment generating function $\psi_t(u)$ of $N(t)$ satisfies 
\begin{align*}
\psi_t(u) & = \E\{ e^{uN(t)} \} = \E\{ e^{u\xi(1) + \cdots + u\xi(t)} \} \\
& = \E\Bigl\{ \prod_{s=1}^t \E( e^{u\xi(s)} \mid \A_{s-1} ) \Bigr\} = \E\Bigl\{ \prod_{s=1}^t \bigl[ 1 - (1-e^u) \pi(s-1) \bigr] \Bigr\}. 
\end{align*}
But \citet{rs1996} show that $\omega_t \leq \min\{\pi(1),\ldots,\pi(t)\}$, so 
\[ \psi_t(u) \leq \prod_{s=1}^t \bigl[ 1-(1-e^u) \omega_t \bigr] = \bigl[1-(1-e^u)\omega_t \bigr]^t. \]
But the right-hand side above is exactly $\phi_t(u)$, completing the proof.  
\end{proof}

Back to our main discussion, we now have that
\[ \prob\Bigl\{ \bigcup_{t \geq T} B(t)^c \Bigr\} \leq \sum_{t \geq T} \psi_t(-h) \leq \sum_{t \geq T} \phi_t(-h), \]
and it is well-known that the moment generating function $\phi_t$ for the binomial random variable satisfies
\[ \phi_t(-h) = \bigl[ 1 - \omega_t (1-e^{-h}) \bigr]^t = \bigl[ 1 - \pi_1(0)(1-e^{-h}) \theta^{\gamma(t)} \bigr]^t. \]
But the sequence $\gamma(t)$ grows like $\ln(t)$, and $\theta^{\ln(t)} = t^{\ln(\theta)}$, so for large $t$
\[ \phi_t(-h) \sim \Bigl[ 1 - \frac{\pi_1(0)(1-e^{-h})}{t^{-\ln(\theta)}} \Bigr]^t. \]
The right-hand side above is just the sequence $\zeta(t)$ defined in Lemma~\ref{lem:sum}, with 
\[ a = \pi_1(0) (1-e^{-h}) \quad \text{and} \quad b = -\ln(\theta). \]
Therefore, the series $\sum_{t \geq 1} \phi_t(-h)$ converges so, for any $\delta$, there exists $T$ such that $\sum_{t \geq T} \phi_t(-h) < \delta$, thus proving \eqref{eq:goal}.  To put everything together, let $T_4^\star$ be the smallest $T$ with $\sum_{t \geq T} \phi_t(-h) < \delta$.  Then Theorem~\ref{thm:eps.opt} follows by taking $T^\star = T_3^\star + T_4^\star$.  

A natural question is if one can give a deterministic bound for $T^\star$ in terms of the user-specfied $\eps$, $\delta$, and $\theta$.  An affirmative answer to this question is given in Appendix~\ref{S:iterations}.  We choose not to give much emphasis to this result, as the bound we obtain appears to be quite conservative.  For example, for numerical experiments run under the setup in Simulation~1 of \citet{thathachar.sastry.1985}, we find that more than 95\% of sample paths converge in roughly 25--250 iterations, while our conservative theoretical bounds are, for moderate $\theta$, orders of magnitude greater.  

\section{Discussion}
\label{S:discuss}

In this paper we have taken a closer look at convergence properties of pursuit learning.  In particular, we have identified a gap in existing proofs of $\eps$-optimality and provided a new argument to fill this gap.  An important consequence of our theoretical analysis is that it seems necessary to explicitly specify the rate at which the tuning parameter sequence $\lambda=\lambda_t$ vanishes with $t$.  In fact, if $\omega_t$ defined in Lemma~\ref{lem:mgf} vanishes too quickly, which it would if $\lambda_t \equiv \lambda$, then $\sum_t \phi_t(-h) = \infty$ and the proof fails.  But we should also reiterate that a theoretical analysis that requires vanishing $\lambda_t$ need not conflict with the tradition of running the algorithm with fixed small $\lambda$ in practical applications.  In fact, the particular $\lambda_t$ vanishes relatively fast so, for applications, we recommend running pursuit learning with, say, 
\[ \lambda_t = 1 - \theta^{[1 + (t-t_0)^+]^{-1}}, \]
where $t_0$ is some fixed cutoff, and $x^+ = \max\{0,x\}$.  This effectively keeps $\lambda_t$ constant for a fixed finite period of time, after which it vanishes like that in \eqref{eq:lambda}.  Alternatively, one might consider $\lambda_t = 1-\theta^{v(t)}$, where $v(t)$ vanishes more slowly than $t^{-1}$.  This choice of $\lambda_t$ vanishes more slowly than that in \eqref{eq:lambda}, thus giving the algorithm more opportunities to adjust to the environment.  We believe that an analysis similar to ours can be used to show that the corresponding pursuit learning converges in the sense of Definition~\ref{def:eps.opt}. 

It is also worth mentioning that the results of \citet{tmm2010}, for $\lambda_t$ as in \eqref{eq:lambda}, can be applied to show that $\pi_1(t) \to 1$ with probability~1 as $(t,\lambda_t) \to (\infty,0)$.  This, of course, immediately implies $\eps$-optimality in the sense of Definition~\ref{def:eps.opt}.  However, this indirect argument does not give any insight as to how to bound the number of iterations needed to be sufficiently close to convergence, as we do---albeit conservatively---in Appendix~\ref{S:iterations}.  But the result proved in \citet{tmm2010} that $\dhat_1(t)$ is largest among the $\dhat(t)$'s infinitely often with probability~1, together with the formula in \eqref{eq:pi.monotone} in Appendix~\ref{S:iterations} can perhaps be used to reason towards and almost sure rate of convergence for pursuit learning.

\section*{Acknowledgements}

A portion of this work was completed while the first author was affiliated with the Department of Mathematical Sciences at Indiana University--Purdue University Indianapolis.  The authors thank Professors Snehasis Mukhopadhyay and Chuanhai Liu for sharing their insight, and the Editor and two anonymous referees for some helpful suggestions.

\appendix

\section{Proof of Lemma \ref{lem:sum}}
\label{S:proof}

To start, write $\zeta(t)$ as
\[ \zeta(t) = \Bigl(1-\frac{a}{t^b}\Bigr)^t = \Bigl[ \Bigl(1-\frac{a}{t^b}\Bigr)^{t^b} \Bigr]^{t^{1-b}}. \]
If $f(t) = (1-\frac{a}{t})^t$, then ordinary calculus reveals that
\[ \frac{d}{dt} \ln f(t) = \ln\Bigl(1-\frac{a}{t}\Bigr) + \frac{a}{t-a} = -\ln\Bigl(1+\frac{a}{t-a}\Bigr) + \frac{a}{t-a} > 0. \]
Therefore, we have shown that $\ln f(t)$ and, hence, $f(t)$ and, hence, $f(t^b)$ are monotone increasing.  Moreover, $f(t^b) \uparrow e^{-a} < 1$.  Thus, $\zeta(t) \leq \exp\{-at^{1-b}\}$.  So to show that $\sum_{t=1}^\infty \zeta(t)$ is finite, it suffices to show that, for $c=1-b$,
\[ \int_1^\infty e^{-at^c} \,dt < \infty. \]
Making a change-of-variable $x=t^c$, the integral becomes
\[ \int_1^\infty e^{-at^c} \,dt = \int_1^\infty \frac{1}{cx^{1-1/c}} e^{-ax} \,dx = \frac1c \int_1^\infty x^{1/c-1} e^{-ax} \,dx. \]
Since $1/c > 1$ and $a > 0$, the integral is finite, completing the proof.

Making one more change-of-variables ($y=ax$), one finds that the last integral above can be expressed as 
\[ \int_1^\infty e^{-at^c}\,dt = \frac{1}{ca^{1/c}} \int_a^\infty y^{1/c-1} e^{-y} \,dy = \frac{1}{ca^{1/c}} \, \Gamma(c^{-1}; a), \]
where $\Gamma(s,x) = \int_x^\infty u^{s-1} e^{-u} \,du$ is the incomplete gamma function.

\section{A bound on the number of iterations}
\label{S:iterations}

As a follow-up to the proof of Theorem~\ref{thm:eps.opt}, we give a conservative upper bound on the number of iterations $T^\star$ needed to be sufficiently close to convergence.  

\begin{theorem}
\label{thm:iterations}
For given $\eps, \delta \in (0,1)$ and $\theta \in (e^{-1},1)$, a deterministic bound on the necessary number of iterations $T^\star$ in Theorem~\ref{thm:eps.opt} can be found numerically.  
\end{theorem}

In the proof that follows, we are assuming $h$ to be a known constant, while it actually depends on the $\eta$ used above which, in turn, depends on the unknown $d$'s.  The bounds obtained in \citet{rs1996} also depend on $\eta$, called the size of the problem.  To use this bound in practice, users must estimate $\eta$ by some other means.    

\begin{proof}[Proof of Theorem~\ref{thm:iterations}]
As stated above, the desired $T^\star$ is actually a sum $T_3^\star + T_4^\star$.  Let's begin with $T_4^\star$, the smallest $T=T(\delta)$ such that $\sum_{t \geq T} \phi_t(-h) < \delta$.  From the proof of the classical integral test for convergence of infinite series in calculus, it follows that 
\[ \sum_{t \geq T} \phi_t(-h) \leq \phi_T(-h) + \int_T^\infty \phi_t(-h) \,dt. \]
A modification of the argument presented in Appendix~\ref{S:proof} shows that 
\[ \int_T^\infty \phi_t(-h) \,dt \leq \frac{b}{a^b} \, \Gamma(b; aT), \]
where $a=\pi_1(0)(1-e^{-h})$, $b=(1+\ln\theta)^{-1}$, and $\Gamma(s,x)$ is the incomplete gamma function (defined in Appendix~\ref{S:proof}).  Since $\phi_T(-h)$ and $\Gamma(b; aT)$ are both decreasing functions of $T$, it is possible to solve the equation 
\[ \phi_T(-h) + \frac{b}{a^b} \, \Gamma(b; aT) = \delta \]
for $T$ numerically to obtain the bound $T_4^\star$ in terms of the user-specified inputs.  

Towards bounding $T_3^\star$ we note that $\dhat_1(t+T_4^\star)$ is the largest of the $\dhat$'s for all $t \geq 1$ for sample paths in a set $\Omega$ of probability $> 1-\delta$.  For sample paths in this $\Omega$, \citet{tmm2010} prove that 
\begin{equation}
\label{eq:pi.monotone}
\pi_1(t+T_4^\star) = 1 - \theta^{\gamma(t+T_4^\star)}\{1 - \pi_1(T_4^\star)\}, \quad t \geq 1, 
\end{equation}
where $\gamma(t) = \sum_{s=1}^t s^{-1}$.  Since $\pi_1(T_4^\star) \in (0,1)$, it easily follows that 
\[ 1-\pi_1(t+T_4^\star) \leq \theta^{\gamma(t+T_4^\star)}, \quad t \geq 1. \]
Given $\eps$ and $T_4^\star$, it is easy to calculate $T_3^\star$ such that $\theta^{\gamma(t+T_4^\star)} \leq \eps$ for all $t > T_3^\star$.   
\end{proof}

\bibliographystyle{/Users/rgmartin/Research/TexStuff/asa}
\bibliography{/Users/rgmartin/Research/mybib}

\end{document}